\documentclass[11pt]{article}
\usepackage[margin=1in]{geometry}
\usepackage{cite}
\usepackage{booktabs}
\usepackage[font=footnotesize,labelfont=normalfont,textfont=normalfont,labelsep=period,justification=raggedright,singlelinecheck=false]{caption}
\usepackage{xspace}
\usepackage{amsmath,amssymb,amsfonts}
\usepackage{algorithm}
\usepackage{algpseudocode}
\usepackage{graphicx}
\usepackage{textcomp}
\usepackage{xcolor}
\usepackage{amsthm}
\usepackage{url}
\usepackage{tikz}
\usetikzlibrary{positioning,arrows.meta,fit,backgrounds,calc,decorations.pathreplacing,topaths,hobby,cd}
\usepackage{pgfplots}
\pgfplotsset{compat=1.18}
\usepgfplotslibrary{groupplots}

\newcommand{\minb}[1]{\hat{#1}}

\newcommand{\fd}{\mathrm{fd}}

\renewcommand\epsilon\varepsilon
\renewcommand\phi\varphi

\newcommand{\RR}{\mathbb{R}}
\newcommand{\NN}{\mathbb{N}}
\newcommand{\ZZ}{\mathbb{Z}}

\newtheorem{proposition}{Proposition}[section]

\newtheorem{theorem}{Theorem}[section]

\newtheorem{definition}{Definition}[section]
\newtheorem{lemma}{Lemma}[section]

\newtheoremstyle{uprightthm}%
  {\topsep}   
  {\topsep}   
  {\normalfont}  
  {}          
  {\bfseries} 
  {.}         
  { }         
  {}          
\theoremstyle{uprightthm}
\newtheorem{example}{Example}[section]
\usepackage{tcolorbox}
\tcbuselibrary{breakable,skins}
\tcolorboxenvironment{example}{
  enhanced jigsaw, breakable,
  colback=white, colframe=black!60, boxrule=0.4pt, sharp corners,
  left=4pt, right=4pt, top=3pt, bottom=3pt,
  before skip=8pt plus 2pt, after skip=8pt plus 2pt
}

\usepackage{microtype}
\usepackage{soul}
\newif\ifchg \chgfalse
\newcommand{\chg}[1]{\ifchg{\sethlcolor{yellow}\hl{#1}}\else#1\fi}
\soulregister\cite7 \soulregister\ref7 \soulregister\cref7
\usepackage[colorlinks=true,linkcolor=blue!50!black,citecolor=blue!50!black,urlcolor=blue!50!black]{hyperref}
\newcommand{\netgraphwidth}{0.62\textwidth}
\newcommand{\mlbr}{}
\newcommand{\mleq}{=}
\renewenvironment{multline*}{\csname gather*\endcsname}{\csname endgather*\endcsname}
\newenvironment{IEEEkeywords}%
  {\par\medskip\noindent\textbf{Keywords: }}{\par\medskip}

\begin{document}
\bibliographystyle{IEEEtran}

\title{\texorpdfstring{\chg{M-Fibration Theory with Applications to Weighted Graphs}}{M-Fibration Theory with Applications to Weighted Graphs}}
\author{Paolo Boldi\\[-1pt]
  \small Dipartimento di Informatica, Universit\`a degli Studi di Milano, Milan, Italy\\
  \small \texttt{paolo.boldi@unimi.it}
  \and \chg{Osvaldo M. Velarde}\\[-1pt]
  \small \chg{Levich Institute and Physics Department,}\\[-2pt]
  \small \chg{City College of New York, New York, NY, USA}
  \and \chg{Hern\'an A. Makse}\\[-1pt]
  \small \chg{Levich Institute and Physics Department,}\\[-2pt]
  \small \chg{City College of New York, New York, NY, USA}}
\date{}

\maketitle

\begin{abstract}
  The purpose of this paper is to provide a general, comprehensive, theoretical framework that allows one to deal with fibrations on graphs labelled on a commutative monoid. This is a genuine extension of the theory of graph fibrations (as introduced in~\cite{BoVGF}), that makes it possible to deal with weighted graphs, and also graphs labelled with other algebraic structures. The derived theory also lends itself naturally to consider approximate fibrations. \chg{As an example, we show how the derived theory can be applied to the reduction of weighted networks, providing a strong theoretical underpinning to recent empirical results~\cite{VePRFS,VePHM}.}
\end{abstract}


\section{Introduction and Motivations}

Graph fibrations were formally described in~\cite{BoVGF}, derived from the much more general definition by Grothendieck~\cite{GroTDE}. The main motivation, at the time, was to have a formal tool to describe computations in distributed systems. Similar notions had already been used in very different settings, with different goals and under different names: equitable partitions~\cite{SchCCP,GoRAGT}, the Weisfeiler--Leman test~\cite{WeLRGC}, colour refinement~\cite{CoGEAG}, groupoids~\cite{GoSDBN}.

While fibrations are still used in the context of distributed computing (see, e.g.,~\cite{DiVOCA}), in the last decades the fibration machinery was borrowed in other, quite different, playgrounds as means to describe symmetries in general networks: specifically, in the theory of dynamical systems~\cite{GoSTPS,GoSDBN,JolGBS} and in biology~\cite{MoLFSU,StRDBG,MBFSbook}.

Although the theory of fibrations is elegant and simple, there are two important shortcomings that make its usage difficult  in these more general contexts. The first is that it is challenging to adapt fibrations to weighted graphs, or more generally to graphs with labels that can be algebraically composed; the theory, as described in~\cite{BoVGF}, allows for labels, but treats them in a ``static'' way: labels can exist on arcs, but they must be preserved exactly by morphisms. The second problem is that the presence of errors or noise is incompatible with the theory itself. One extra, or one missing arc, or even an arc with a weight that is slightly off, and everything breaks.

These problems can be overcome with \emph{ad hoc} approaches, and in fact fibrations have been successfully applied also in these more general contexts~\cite{StRDBG,BoLMQF,VePRFS}. Nonetheless, a well-grounded theoretical foundation behind their use is still missing.

\smallskip
The purpose of this paper is to provide a general, comprehensive, theoretical framework that allows one to deal with fibrations on graphs labelled on a commutative monoid. The idea is similar to that of weighted bisimulation of weighted automata~\cite{BucBRW}, and akin to exact lumpability of Markov chains~\cite{KeSFMC,BucEOL}, which is the special case in which labels are non-negative reals under addition; other instances of the same idea appear as equitable partitions of weighted graphs and their role in the spectral theory of graphs~\cite{SchCCP,GoRAGT} and, more recently, in the analysis of the expressive power of message-passing graph neural networks~\cite{MoRFWL,XuHLJ}, and as a tool for the lossless dimension reduction of linear programs~\cite{GrKMS}. What these approaches have in common is a partition-refinement algorithm~\cite{PaiTPR}; what they lack is a common conceptual framework in which the objects being computed (bases and fibrations, rather than partitions) are first-class citizens, together with their universal properties. Providing such a framework is the purpose of this paper.

As an example of application, we show how the derived theory can be applied to the compression of arbitrary neural networks (including CNNs), providing a strong theoretical underpinning to the recent results in~\cite{VePRFS}.

\section{Notations and Preliminaries}\label{sec:notations}

\paragraph{Equivalence relations and partitions.} An \emph{equivalence relation} on a set $X$ is a binary relation $\sim$ that is symmetric, reflexive and transitive. A \emph{partition} of a set $X$ is a collection $\Pi$ of pairwise-disjoint non-empty subsets of $X$ whose union is $X$ itself; the elements of $\Pi$ are called the \emph{parts} of the partition.
Partitions and equivalence relations are in one-to-one correspondence: given an equivalence relation $\sim$ on $X$, the set of equivalence classes of $\sim$ is a partition of $X$, denoted by $\Pi_\sim$, and given a partition $\Pi$ of $X$, the relation $\sim_\Pi$ defined by $x\sim_\Pi y$ if and only if $x$ and $y$ belong to the same part of $\Pi$ is an equivalence relation on $X$.

An equivalence relation (or partition) $\sim_1$ is said to be \emph{finer} than an equivalence relation $\sim_2$ (or, equivalently, $\sim_2$ is \emph{coarser} than $\sim_1$) if $x\sim_1 y$ implies $x\sim_2 y$ for all $x,y\in X$. 

\paragraph{Graphs and graph morphisms.} A \emph{graph} $G$ is defined by $(N_G,A_G,s_G,t_G)$, where $N_G$ is a set of \emph{nodes}, $A_G$ is a set of \emph{arcs}, and $t_G,s_G: A_G \to N_G$ provide the source and target node of each arc; in this paper, we only consider finite graphs, i.e., graphs with a finite set of nodes and arcs. In the following, all subscripts are omitted when they are clear from the context.

We write $G(x,y)$ for the set of arcs from $x$ to $y$, i.e., $G(x,y)=s^{-1}(x)\cap t^{-1}(y)$, and extend this notation to sets of nodes $X,Y\subseteq N_G$ by writing $G(X,Y)=\bigcup_{x\in X,y\in Y} G(x,y)$. We also use $-$ instead of $N_G$, so for instance $G(-,y)$ is the set of arcs ending in $y$.

A \emph{simple} graph is a graph with no parallel arcs, i.e., such that $|G(x,y)|\le 1$ for all $x,y\in N_G$.

A \emph{graph morphism} $f: G \to H$ is a pair of functions $f_N: N_G \to N_H$ and $f_A: A_G \to A_H$ such that the following diagrams commute:
\[
\begin{tikzcd}
  A_G \arrow[r, "f_A"] \arrow[d, "s_G"'] & A_H \arrow[d, "s_H"] & A_G \arrow[r, "f_A"] \arrow[d, "t_G"'] & A_H \arrow[d, "t_H"] \\
  N_G \arrow[r, "f_N"] & N_H & N_G \arrow[r, "f_N"] & N_H
\end{tikzcd}
\]
A graph morphism $f: G \to H$ is \emph{surjective} (also called \emph{epimorphic}) (\emph{injective}, \emph{bijective}, resp.) if both $f_N$ and $f_A$ are surjective (injective, bijective, resp.) functions.

\section{$M$-graphs and $M$-fibrations}

A (standard) \emph{graph fibration}~\cite{BoVGF} is a graph morphism $\phi: G \to H$ such that for every arc $\bar a\in A_H$ and node $x\in \phi^{-1}(t(\bar a))$ there exists a unique $a\in \phi^{-1}(\bar a)\cap G(-,x)$. This condition is extremely rigid: while the whole theory of graph fibrations is very elegant and has many applications, it is not flexible enough to be applied to weighted or labelled neural networks, unless we just assume that morphisms simply preserve the labels/weights, which is usually something we don't want. In order to overcome this limitation, we introduce the notion of $M$-fibrations: let us start with the definition of $M$-graph.

\begin{definition}[$M$-graph] 
  Given a commutative monoid $(M,\oplus,0)$, an \emph{$M$-graph} is a graph $G$ together with a function $\lambda_G: A_G \to M$ that assigns a label in $M$ to each arc of $G$. 
\end{definition}

From now on, $M$ will always denote a commutative monoid (sometimes enjoying further properties). 
In Figure~\ref{fig:example}, you can see two (simple) $M$-graphs: the monoid $M$ is intentionally left unspecified. Note that here the underlying graph is the same, but the labelling is not.

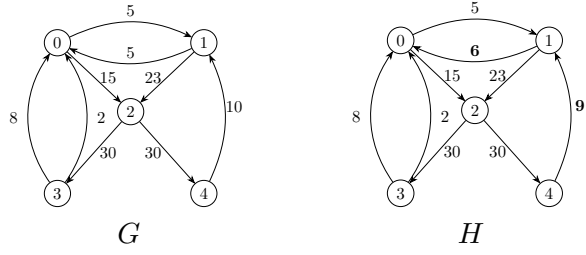
\begin{figure}
  \centering
    \begin{tabular}{ccc}
      \resizebox{0.2\textwidth}{!}{
    \begin{tikzpicture}[
      n/.style={circle, draw, fill=white, minimum size=6pt, inner sep=3pt},
      >={Stealth[length=5pt]}
    ]
    \node[n] (n0)    at (-1.7,2.2)     {$0$};
    \node[n] (n1)    at (1.7,2.2)     {$1$};
    \node[n] (n2) at (0,0.6)     {$2$};
    \node[n] (n3)     at (-1.7,-1.3) {$3$};
    \node[n] (n4)     at ( 1.7,-1.3) {$4$};

    \draw[->] (n0) to[bend left=25] node[pos=0.5,above] {$5$} (n1);
    \draw[->] (n1) to[bend left=25] node[pos=0.5,above] {$5$} (n0);
    \draw[->] (n0) -- (n2) node[midway,right] {15};
    \draw[->] (n1) -- (n2) node[midway,left] {23};
    \draw[->] (n2) -- (n3) node[midway,right] {30};
    \draw[->] (n2) -- (n4) node[midway,left] {30};
    \draw[->] (n3) to[bend left=35]  node[pos=0.5, left, xshift=-3pt] {$8$} (n0);
    \draw[->] (n3) to[bend right=35]  node[pos=0.5, right, xshift=3pt] {$2$} (n0);
    \draw[->] (n4) to[bend right=25] node[pos=0.5, above, xshift=6pt] {$10$} (n1);
  \end{tikzpicture}} & $\quad$ &
      \resizebox{0.2\textwidth}{!}{
    \begin{tikzpicture}[
      n/.style={circle, draw, fill=white, minimum size=6pt, inner sep=3pt},
      >={Stealth[length=5pt]}
    ]
    \node[n] (n0)    at (-1.7,2.2)     {$0$};
    \node[n] (n1)    at (1.7,2.2)     {$1$};
    \node[n] (n2) at (0,0.6)     {$2$};
    \node[n] (n3)     at (-1.7,-1.3) {$3$};
    \node[n] (n4)     at ( 1.7,-1.3) {$4$};

    \draw[->] (n0) to[bend left=25] node[pos=0.5,above] {$5$} (n1);
    \draw[->] (n1) to[bend left=25] node[pos=0.5,above] {\bf 6} (n0);
    \draw[->] (n0) -- (n2) node[midway,right] {15};
    \draw[->] (n1) -- (n2) node[midway,left] {23};
    \draw[->] (n2) -- (n3) node[midway,right] {30};
    \draw[->] (n2) -- (n4) node[midway,left] {30};
    \draw[->] (n3) to[bend left=35]  node[pos=0.5, left, xshift=-3pt] {$8$} (n0);
    \draw[->] (n3) to[bend right=35]  node[pos=0.5, right, xshift=3pt] {$2$} (n0);
    \draw[->] (n4) to[bend right=25] node[pos=0.5, above, xshift=6pt] {\bf 9} (n1);
  \end{tikzpicture}} \\
      $G$ && $H$
    \end{tabular}
  \caption{\label{fig:example}Two graphs with labels on some monoid (for the moment, we do not specify the monoid operations). The underlying graph is the same, but the labelling is slightly different (in boldface, the labels that change).}
\end{figure}

\begin{definition}[$M$-fibration] 
  An \emph{$M$-fibration} is a graph morphism $\phi: G \to H$ between two $M$-graphs such that for every arc $\bar a\in A_H$ and node $x\in \phi^{-1}(t(\bar a))$ we have
  \[
      \lambda(\bar a)=\bigoplus_{a\in \phi^{-1}(\bar a)\cap G(-,x)} \lambda(a).
  \]
\end{definition}

It is easy to see that standard graph fibrations $\phi: G \to B$ are a special case of $M$-fibrations, where the monoid is $(\mathbb{N},+,0)$ and all arcs in both $G$ and $B$ have label $1$. In this case, the condition for being an $M$-fibration reduces to the condition for being a standard graph fibration (the set over which the sum is taken must have exactly one element).

Some basic properties of $M$-fibrations are the following:
\begin{proposition}
  \label{prop:comp-bij}
  \begin{enumerate}
    \item\label{enu:comp} The composition of two $M$-fibrations is an $M$-fibration.
    \item\label{enu:bij} A bijective $M$-fibration is an isomorphism of $M$-graphs (i.e., a bijective graph morphism that preserves the labels of arcs).
  \end{enumerate}
\end{proposition}
\begin{proof}
  \noindent\ref{enu:comp}. Let $\phi: G \to H$ and $\psi: H \to K$ be two $M$-fibrations. Let $\bar a\in A_K$ and $x\in (\psi\circ \phi)^{-1}(t(\bar a))$. Then we have:
  \begin{multline*}
    \bigoplus_{a\in (\psi\circ \phi)^{-1}(\bar a)\cap G(-,x)} \lambda(a)=\mlbr
    \bigoplus_{b\in \psi^{-1}(\bar a)\cap H(-,\phi(x))} 
    \bigoplus_{a\in \phi^{-1}(b)\cap G(-,x)}
    \lambda(a)=\mlbr
    \bigoplus_{b\in \psi^{-1}(\bar a)\cap H(-,\phi(x))} \lambda(b)=\lambda(\bar a).
  \end{multline*}

  \noindent\ref{enu:bij}. Let $\phi: G \to H$ be a bijective $M$-fibration. Then for every arc $\bar a\in A_H$ and node $x\in \phi^{-1}(t(\bar a))$ we have 
  \[
      \lambda(\bar a)=\bigoplus_{a\in \phi^{-1}(\bar a)\cap G(-,x)} \lambda(a).
  \]
  But since $\phi$ is bijective, there is exactly one node $x \in \phi^{-1}(t(\bar a))$ and the set $\phi^{-1}(\bar a)\cap G(-,x)$ has exactly one element, so the last formula reduces to $\lambda(\bar a)=\lambda(\phi^{-1}(\bar a))$.
\end{proof}

\section{Minimum base of an $M$-graph}\label{sec:min-base}

The theory of graph fibrations~\cite{BoVGF} ensures that every graph has a minimum base, i.e., for every graph $G$ there exists an epimorphic fibration $\phi: G \to B$ whose base $B$ is fibration prime (it cannot be fibred non-trivially), and $B$ is unique up to isomorphism; moreover, all such minimal fibrations have the same node component (up to isomorphism of the base), although they may differ on arcs. In this section we prove that the same result holds for $M$-fibrations, in fact in a stronger form: since minimum bases turn out to be simple graphs, minimal $M$-fibrations are unique up to isomorphism.

\begin{definition}[$M$-equitable partition]
  Given an $M$-graph $G$, an \emph{$M$-equitable partition} of $G$ is a partition $\Pi$ of $N_G$ such that for every two nodes $x,y\in N_G$ such that $x \sim_\Pi y$ and every part $C\in \Pi$ we have
  \[
      W(C,x)=W(C,y)
  \]
  where $W(A,z)=\bigoplus_{a\in G(A,z)} \lambda(a)$ is the total weight of the arcs from $A$ to $z$. When we have an equitable partition, given two parts $C,D\in \Pi$, we write $W(C,D)$ for the common value of $W(C,x)$ for all $x\in D$.
\end{definition}

The following proposition shows that $M$-fibrations and $M$-equitable partitions are two sides of the same coin: given an $M$-fibration, the equivalence relation on nodes defined by the fibers of the fibration is an equitable partition, and given an equitable partition, the canonical projection to the quotient graph is an $M$-fibration.

\begin{proposition}
  \label{prop:fib-equi}
  \begin{enumerate}
    \item \label{enu:fib-equi} Let $\phi: G \to H$ be an epimorphic $M$-fibration: then the partition $\Pi_\sim$ associated with the equivalence relation $\sim$ defined by $x\sim y$ if and only if $\phi(x)=\phi(y)$ is an $M$-equitable partition of $G$.
    \item \label{enu:equi-fib} Let $\Pi$ be an $M$-equitable partition of $G$, and let $G/\Pi$ be the simple $M$-graph obtained as the quotient of $G$ by $\Pi$ (its nodes are the parts of $\Pi$, and there is an arc $(C,D)$, labelled by $W(C,D)$, if and only if $G(C,D)\neq\emptyset$); let $\pi: G \to G/\Pi$ be the canonical projection to the quotient graph then $\pi$ is an epimorphic $M$-fibration.
  \end{enumerate}
\end{proposition}
\begin{proof}
  \noindent\ref{enu:fib-equi}. Let $x\in N_G$ and let $C=\phi^{-1}(\bar z)$ for some $\bar z \in N_H$. 
  Then, by definition,
  \[
  W(C,x)=\bigoplus_{a\in G(C,x)} \lambda(a). 
  \]
  The image under $\phi$ of each $a$ appearing in the last formula is an arc of $H$ starting from $\bar z$ and ending in $\phi(x)$; conversely, for every $\bar a\in H(\bar z,\phi(x))$, every $a\in\phi^{-1}(\bar a)\cap G(-,x)$ belongs to $G(C,x)$ (its source is mapped to $\bar z$). Hence the sets $\phi^{-1}(\bar a)\cap G(-,x)$, for $\bar a\in H(\bar z,\phi(x))$, partition $G(C,x)$, and we can rewrite the last formula factorizing by the image of $a$ under $\phi$:
  \begin{multline*}
    W(C,x)=\bigoplus_{a\in G(C,x)} \lambda(a)\mleq
    \bigoplus_{\bar a \in H(\bar z,\phi(x))} \bigoplus_{a\in \phi^{-1}(\bar a) \cap G(-,x)} \lambda(a)=\bigoplus_{\bar a \in H(\bar z,\phi(x))} \lambda(\bar a),
  \end{multline*}
  where the last equality follows from the definition of $M$-fibration.
  The last expression does not depend on $x$ but only on $\phi(x)$, so if $x\sim y$ we have $W(C,x)=W(C,y)$, as desired.

  \noindent\ref{enu:equi-fib}. The only non-trivial part is proving that $\pi$ is an $M$-fibration. Let $(C,D)$ be an arc of $G/\Pi$ and let $x\in \pi^{-1}(D)$. Then
  \begin{multline*}
    \bigoplus_{a\in \pi^{-1}(C,D)\cap G(-,x)} \lambda(a)=\bigoplus_{a\in G(C,x)} \lambda(a)\mleq W(C,x)=W(C,D)=\lambda(C,D).
  \end{multline*}
  which is the condition required for $\pi$ to be an $M$-fibration.
\end{proof}

The next lemma is crucial in proving the existence of a minimum base for $M$-fibrations.
\begin{lemma}
  \label{lemma:coarsest-equi}
  Let $G$ be an $M$-graph. Then there exists a coarsest $M$-equitable partition of $G$, i.e., an $M$-equitable partition $\Pi$ such that for every $M$-equitable partition $\Pi'$ of $G$ we have $\sim_{\Pi'}$ is finer than $\sim_{\Pi}$.
\end{lemma}
\begin{proof}
  We build a sequence of partitions $\Pi_0,\Pi_1,\ldots$ of $N_G$ as follows: $\Pi_0$ is the trivial partition with a single part $N_G$, and for $i>0$ we define $\Pi_i$ as the partition of $N_G$ whose parts are the equivalence classes of the relation $\sim_i$ defined by
  \[
    x\sim_i y \Leftrightarrow x\sim_{i-1} y \text{ and } W(C,x)=W(C,y) \quad \forall C\in \Pi_{i-1}.
  \]
  This sequence of finer and finer partitions must eventually stabilize (because the set of nodes is finite), i.e., there exists $k$ such that $\Pi_k=\Pi_{k+1}$ (and hence $\Pi_i=\Pi_k$ for all $i\ge k$, since $\Pi_{i}$ depends only on $\Pi_{i-1}$). Note that $\Pi^*=\Pi_k$ is $M$-equitable: indeed, $\Pi_{k+1}=\Pi_k$ says precisely that $W(C,x)=W(C,y)$ for all $C\in\Pi_k$ whenever $x\sim_k y$. We claim that $\Pi^*$ is the coarsest $M$-equitable partition of $G$.
  Suppose that $\Xi$ is an $M$-equitable partition of $G$. We prove by induction on $i$ that $\Xi$ is finer than $\Pi_i$ for all $i\ge 0$. The base case $i=0$ is trivial. Suppose that $\Xi$ is finer than $\Pi_i$. Let $x,y$ be in the same part of $\Xi$; then they are also in the same part of $\Pi_i$. Let $C$ be a part of $\Pi_i$: since $\Xi$ is finer than $\Pi_i$, there are parts $C_1,\dots,C_m$ of $\Xi$ such that $C=C_1\cup\dots\cup C_m$ (a disjoint union). Since $\Xi$ is an $M$-equitable partition and $x,y$ lie in the same part of $\Xi$, we have $W(C_j,x)=W(C_j,y)$ for every $j$, whence
  \[
    W(C,x)=\bigoplus_{j=1}^m W(C_j,x)=\bigoplus_{j=1}^m W(C_j,y)=W(C,y),
  \]
  so $x,y$ are also in the same part of $\Pi_{i+1}$, and we conclude that $\Xi$ is finer than $\Pi_{i+1}$. By induction, $\Xi$ is finer than $\Pi_i$ for all $i\ge 0$, and in particular it is finer than $\Pi^*=\Pi_k$.
\end{proof}
We are now ready to prove that every $M$-graph has a minimum base.

\begin{definition}[$M$-fibration prime and minimal $M$-fibration]
  An $M$-graph $G$ is \emph{$M$-fibration prime} if every epimorphic $M$-fibration $\phi: G \to H$ is an isomorphism. An $M$-fibration $\phi: G \to H$ is \emph{minimal} if it is epimorphic and $H$ is $M$-fibration prime. 
\end{definition}

Note that every $M$-fibration prime graph is simple, i.e., it has no parallel arcs, because if $G$ has two parallel arcs $a,b$ from $x$ to $y$, then $G$ is non-trivially and epimorphically fibred onto the graph $H$ obtained by merging the two arcs with a new arc labelled by $\lambda(a)\oplus \lambda(b)$ (which may be $0$: this is why labels are allowed to be $0$).

As a consequence of the above results, we have the following theorem, which is the main result of this section (and the equivalent, in the context of $M$-fibrations, of the well-known result that every graph has a minimum base~\cite{BoVGF}):
\begin{theorem}
  \label{thm:min-base}
  For every $M$-graph $G$, we have:
  \begin{enumerate}
    \item\label{enu:min-fib} there exists a minimal $M$-fibration $\phi: G \to B$;
    \item\label{enu:uniq} $B$ is unique up to isomorphism;
    \item\label{enu:phi-uniq} $\phi$ is unique up to composition with an isomorphism of $B$.
  \end{enumerate}
  In the following, we call $B$ the\footnote{We use the definite article, because of its uniqueness up to isomorphism.} \emph{minimum base} of $G$ and $\phi$ the \emph{minimum $M$-fibration} of $G$, and denote $B$ as $\minb G$ and $\phi$ as $\mu_G: G \to \minb G$.
\end{theorem}
\begin{proof}
  \noindent\ref{enu:min-fib}. Let $\Pi$ be the coarsest $M$-equitable partition of $G$, which exists by Lemma~\ref{lemma:coarsest-equi}. Let $B=G/\Pi$ be the quotient graph and let $\pi: G \to B$ be the canonical projection onto the quotient graph. By Proposition~\ref{prop:fib-equi} (\ref{enu:equi-fib}), $\pi$ is an epimorphic $M$-fibration. We claim that $B$ is $M$-fibration prime. Suppose that $\psi: B \to H$ is an epimorphic $M$-fibration. Then, by Proposition~\ref{prop:comp-bij} (\ref{enu:comp}), $\psi\circ \pi: G \to H$ is an epimorphic $M$-fibration, and by Proposition~\ref{prop:fib-equi} (\ref{enu:fib-equi}), it induces an $M$-equitable partition of $G$. The parts of this partition are unions of parts of $\Pi$ (the fibres of $\psi\circ\pi$ are unions of fibres of $\pi$), so it is coarser than $\Pi$; but $\Pi$ is the coarsest $M$-equitable partition of $G$, so the two partitions coincide, and $\psi$ is injective (hence, being epimorphic, bijective) on nodes. It is then also a bijection on arcs: it is surjective by hypothesis, and if $\psi(a)=\psi(a')$ then $a$ and $a'$ have the same source and the same target (by injectivity on nodes), so $a=a'$ because $B$ is simple. Thus $\psi$ is a bijective $M$-fibration, hence an isomorphism by Proposition~\ref{prop:comp-bij} (\ref{enu:bij}). So we conclude that $B$ is $M$-fibration prime.

  \noindent\ref{enu:uniq} Suppose $\phi: G \to B_1$ is a minimal $M$-fibration, and let $\Pi_1$ be the $M$-equitable partition of $G$ induced by $\phi$. Since $\Pi$ is the coarsest $M$-equitable partition of $G$, we have that $\Pi_1$ is finer than $\Pi$. Define $\chi: B_1 \to B$ by $\chi(x)=\pi(y)$ for any $y \in \phi^{-1}(x)$: this definition is well-posed because if $y_1,y_2 \in \phi^{-1}(x)$, then $y_1\sim_{\Pi_1} y_2$, and since $\Pi_1$ is finer than $\Pi$, we have $y_1\sim_\Pi y_2$, so $\pi(y_1)=\pi(y_2)$. Both $B_1$ and $B$ are simple ($B_1$ because it is $M$-fibration prime, $B$ by construction), so $\chi$ extends to arcs as follows: if $c$ is the arc of $B_1$ from $w$ to $z$, pick $a\in\phi^{-1}(c)$ (which exists because $\phi$ is epimorphic); then $a$ goes from a node of $\phi^{-1}(w)\subseteq\pi^{-1}(\chi(w))$ to a node of $\phi^{-1}(z)\subseteq\pi^{-1}(\chi(z))$, so $G(\pi^{-1}(\chi(w)),\pi^{-1}(\chi(z)))\neq\emptyset$ and $B$ has an arc from $\chi(w)$ to $\chi(z)$: we let $\chi(c)$ be this arc. By construction $\chi$ is a graph morphism with $\pi=\chi\circ\phi$, and it is epimorphic: on nodes because $\pi$ is, and on arcs because every arc of $B$ is of the form $\pi(a)=\chi(\phi(a))$ for some $a\in A_G$.

  We now show that $\chi$ is an $M$-fibration. Let $\bar c$ be the arc of $B$ from $U$ to $V$ (recall that the nodes of $B$ are the parts of $\Pi$, so $\pi^{-1}(U)=U$ and $\pi^{-1}(V)=V$ as sets of nodes of $G$), and let $z\in\chi^{-1}(V)$; fix $y\in\phi^{-1}(z)$, so that $y\in V$. Since $B_1$ is simple, the arcs of $B_1$ ending in $z$ that are mapped by $\chi$ to $\bar c$ are exactly the arcs $(w,z)$ of $B_1$ with $\chi(w)=U$. For every such arc, the set of arcs of $G$ ending in $y$ and mapped by $\phi$ to $(w,z)$ is $G(\phi^{-1}(w),y)$, so the $M$-fibration condition for $\phi$ reads $\lambda(w,z)=W(\phi^{-1}(w),y)$. On the other hand, if $w\in\chi^{-1}(U)$ but $B_1$ has no arc from $w$ to $z$, then $G(\phi^{-1}(w),y)=\emptyset$ (an arc from $\phi^{-1}(w)$ to $y$ would be mapped by $\phi$ to an arc from $w$ to $z$), so $W(\phi^{-1}(w),y)=0$. Therefore
  \begin{multline*}
    \bigoplus_{c\in\chi^{-1}(\bar c)\cap B_1(-,z)}\lambda(c)
    =\bigoplus_{w\in\chi^{-1}(U)}W(\phi^{-1}(w),y)=\\
    =W\Bigl(\bigcup_{w\in\chi^{-1}(U)}\phi^{-1}(w),\,y\Bigr)
    =W(U,y)\mleq W(U,V)=\lambda(\bar c),
  \end{multline*}
  where we used that the sets $\phi^{-1}(w)$, for $w\in\chi^{-1}(U)$, partition $\pi^{-1}(U)=U$, and that $y\in V$. This is precisely the condition for $\chi$ to be an $M$-fibration.

  Hence $\chi: B_1 \to B$ is an epimorphic $M$-fibration; since $B_1$ is $M$-fibration prime, $\chi$ is an isomorphism, which proves \ref{enu:uniq}.

  \noindent\ref{enu:phi-uniq}. With the notation above, $\pi=\chi\circ\phi$ gives $\phi=\chi^{-1}\circ\pi$: every minimal $M$-fibration of $G$ is obtained from the canonical projection $\pi$ by composition with an isomorphism of its base, and any two minimal $M$-fibrations differ by an isomorphism between their bases.
\end{proof}

In Figure~\ref{fig:minimum-bases} we show the minimum base of one of the graphs in Figure~\ref{fig:example}, when the monoid is $(\RR,+,0)$ (left) or $(\ZZ_{31},+,0)$ (right). Note that the two minimum bases are different, because the two monoids are different.
In both cases, as expected, the minimum base is a simple graph, because every non-simple $M$-graph is epimorphically fibred onto a simple $M$-graph obtained by merging parallel arcs.
\begin{figure}
  \centering
    \resizebox{0.2\textwidth}{!}{
    \begin{tikzpicture}[
      n/.style={circle, draw, fill=white, minimum size=6pt, inner sep=0pt},
      >={Stealth[length=5pt]}
    ]
    \node[n] (top)    at (0,2.2)     {$\{0,1\}$};
    \node[n] (center) at (0,0.6)     {$\{2\}$};
    \node[n] (b)      at (0,-1.3) {$\{3,4\}$};

    \draw[->] (top) -- (center) node[midway,left] {38};
    \draw[->] (center) -- (b) node[midway,left] {30};
    \draw[->] (b) to[bend right=45] node[pos=0.5, above, xshift=4pt] {10} (top);
	\draw[->] (top) to[out=120, in=60, looseness=4, min distance=8mm] node[above] {5} (top);
  \end{tikzpicture} $\qquad$ 
  \begin{tikzpicture}[
      n/.style={circle, draw, fill=white, minimum size=6pt, inner sep=0pt},
      >={Stealth[length=5pt]}
    ]
    \node[n] (top)    at (0,2.2)     {$\{0,1\}$};
    \node[n] (center) at (0,0.6)     {$\{2\}$};
    \node[n] (b)      at (0,-1.3) {$\{3,4\}$};

    \draw[->] (top) -- (center) node[midway,left] {7};
    \draw[->] (center) -- (b) node[midway,left] {30};
    \draw[->] (b) to[bend right=45] node[pos=0.5, above, xshift=4pt] {10} (top);
	\draw[->] (top) to[out=120, in=60, looseness=4, min distance=8mm] node[above] {5} (top);
  \end{tikzpicture} }
  \caption{The minimum base $\hat G$ of the graph $G$ in Figure~\ref{fig:example}, when the monoid is $(\RR,+,0)$ (left) or $(\ZZ_{31},+,0)$ (right). We have named each node of $\hat G$ with the corresponding part of the coarsest equitable partition of $G$. The only difference between the two minimum bases is the label of the arc from $\{0,1\}$ to $\{2\}$, which is $15 \oplus 23$: in $\ZZ$ this gives $15+23=38$, while in $\ZZ_{31}$ it gives $15+23=38 \equiv 7 \pmod{31}$.}
  \label{fig:minimum-bases}
\end{figure}
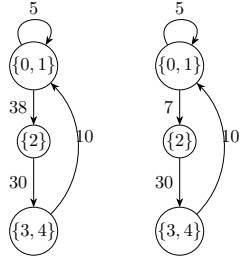

Conversely, the graph $H$ in Figure~\ref{fig:example}, which differs from $G$ only in two labels, has a different minimum base, shown in Figure~\ref{fig:minimum-bases-bis}: the minimum base is sensitive to every label, and a small change of the labels can change it substantially (this is the starting point of Section~\ref{sec:approx}).

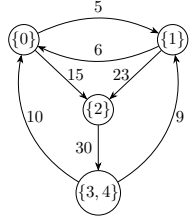
\begin{figure}
  \centering
    \resizebox{0.15\textwidth}{!}{
    \begin{tikzpicture}[
      n/.style={circle, draw, fill=white, minimum size=6pt, inner sep=0pt},
      >={Stealth[length=5pt]}
    ]
    \node[n] (n0)    at (-1.7,2.2)     {$\{0\}$};
    \node[n] (n1)    at (1.7,2.2)     {$\{1\}$};
    \node[n] (center) at (0,0.6)     {$\{2\}$};
    \node[n] (b)      at (0,-1.3) {$\{3,4\}$};

    \draw[->] (n0) to[bend left=25] node[pos=0.5,above] {$5$} (n1);
    \draw[->] (n1) to[bend left=25] node[pos=0.5,above] {$6$} (n0);
    \draw[->] (n0) -- (center) node[midway,right] {$15$};
    \draw[->] (n1) -- (center) node[midway,left] {$23$};
    \draw[->] (center) -- (b) node[midway,left] {$30$};
    \draw[->] (b) to[bend left=35] node[pos=0.5, above, xshift=4pt] {$10$} (n0);
    \draw[->] (b) to[bend right=35] node[pos=0.5, above, xshift=8pt] {$9$} (n1);
  \end{tikzpicture}}
  \caption{The minimum base $\hat H$ of the graph $H$ in Figure~\ref{fig:example}, when the monoid is $(\RR,+,0)$. Again, we have named each node of $\hat H$ with the corresponding part of the coarsest equitable partition of $H$. The reason why $\hat H$ is different from $\hat G$ (see Figure~\ref{fig:minimum-bases}, left) is that $0$ and $1$ cannot be merged (because the two opposite arcs have different labels ($5$ and $6$, instead of $5$ in the two directions), and furthermore the arcs coming from $3$ and $4$ also have different labels ($8+2=10$ and $9$, respectively)). Nodes $3$ and $4$ are merged because the arcs coming from $2$ have the same label $30$.}
  \label{fig:minimum-bases-bis}
\end{figure}

\smallskip
It is worth noting that the proof of Theorem~\ref{thm:min-base} is constructive: it provides a simple algorithm to compute the minimum base of an $M$-graph, by iteratively refining the partition of nodes until it stabilizes. A naive implementation runs in time $O(|A_G|\cdot|N_G|)$; using the ``process the smaller half'' technique of Paige and Tarjan~\cite{PaiTPR}, and assuming that equality in $M$ is decidable in constant time, one obtains $O(|A_G|\log|N_G|)$, as done for $M=(\mathbb{R}_{\ge 0},+,0)$ (Markov-chain lumping) in~\cite{ValFSO}; for standard colour refinement, i.e., $M=(\mathbb{N},+,0)$, this bound is known to be tight~\cite{BeBGTL}. 

As we observed, the minimum base of an $M$-graph is always a simple graph, even if the original graph has parallel arcs. This may seem to contradict the fact that the minimum base of a standard graph fibration may have parallel arcs, but this is not the case: if a standard graph fibration has parallel arcs, then the minimum base is obtained by merging them into a single arc with label equal to the number of merged arcs, which is exactly what happens in the case of $M$-fibrations with $M=(\mathbb{N},+,0)$. 
Another interesting property is that the minimum $M$-fibration is unique (up to isomorphism of the base), while the minimal fibrations of a standard graph may not be unique (even up to isomorphism of the base), because different minimal fibrations may differ on parallel arcs.

\section{Properties of the minimum base of an $M$-graph and pullbacks}

One further property of the minimum base of an $M$-graph is the following:

\begin{theorem}\label{thm:base-invariance}
  For every $M$-graph $G$, if $\phi: G \to H$ is an epimorphic $M$-fibration, then there is an epimorphic $M$-fibration $\psi: H \to \minb G$ such that $\mu_G=\psi\circ\phi$; furthermore $\minb H$ is isomorphic to $\minb G$.
\end{theorem}
\begin{proof}
  The construction of $\chi$ in the proof of Theorem~\ref{thm:min-base}~(\ref{enu:uniq}) never used the fact that $B_1$ is prime, only that it is simple: hence, given an epimorphic $M$-fibration $\phi\colon G\to H$, apply it to $\phi'=\mu_H\circ\phi\colon G\to\minb H$ (an epimorphic $M$-fibration onto a simple graph, by Proposition~\ref{prop:comp-bij}) to get an epimorphic $M$-fibration $\chi\colon\minb H\to\minb G$ with $\mu_G=\chi\circ\mu_H\circ\phi$; then $\psi=\chi\circ\mu_H$ is the required fibration. Moreover $\chi$ is an epimorphic $M$-fibration out of the $M$-fibration prime graph $\minb H$, hence an isomorphism $\minb H\cong\minb G$.
\end{proof}

Now all the properties proved about $M$-fibrations and minimum bases are analogous to (and in fact proper generalizations of) the properties of standard graph fibrations and minimum bases. But the proofs in~\cite{BoVGF} follow a different approach, based on the notion of universal total graphs (which arise from a right adjoint, and depend on the unique lifting of paths) and on the fact that fibrations are preserved by pullbacks. Unfortunately, this route is not available in the case of $M$-fibrations: unique path lifting fails (an arc labelled $2$ may lift to two arcs labelled $1$), and with it the universal property of universal total graphs (Example~\ref{ex:no-universal}); moreover, the category of $M$-graphs and $M$-fibrations does not have pullbacks, in general (Example~\ref{ex:no-pullback}). 

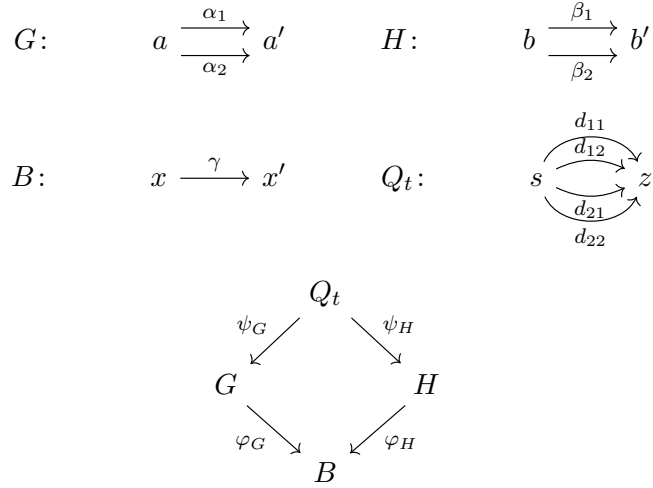
\begin{figure}
\centering
\begin{tikzcd}[column sep=2.4em, row sep=0.3em]
 G\colon & a \arrow[r, shift left=1.1ex, "\alpha_1"] \arrow[r, shift right=1.1ex, "\alpha_2"'] & a'
 & H\colon & b \arrow[r, shift left=1.1ex, "\beta_1"] \arrow[r, shift right=1.1ex, "\beta_2"'] & b'
\end{tikzcd}

\medskip

\begin{tikzcd}[column sep=2.4em, row sep=0.3em]
 B\colon & x \arrow[r, "\gamma"] & x'
 & Q_t\colon & s \arrow[r, bend left=65, "d_{11}"] \arrow[r, bend left=25, "d_{12}"{yshift=-1pt}] \arrow[r, bend right=25, "d_{21}"'{yshift=1pt}] \arrow[r, bend right=65, "d_{22}"'] & z
\end{tikzcd}

\medskip

\begin{tikzcd}[row sep=1.6em, column sep=1.6em]
 & Q_t \arrow[dl, "\psi_G"'] \arrow[dr, "\psi_H"] & \\
 G \arrow[dr, "\phi_G"'] & & H \arrow[dl, "\phi_H"] \\
 & B &
\end{tikzcd}
\caption{Top: the $M$-graphs of Example~\ref{ex:no-pullback}. Labels: $\lambda(\alpha_i)=\lambda(\beta_j)=1$, $\lambda(\gamma)=2$, $\lambda(d_{11})=\lambda(d_{22})=t$, $\lambda(d_{12})=\lambda(d_{21})=1-t$; moreover $\psi_G(d_{ij})=\alpha_i$ and $\psi_H(d_{ij})=\beta_j$. Bottom: the cone $Q_t$ over the cospan $G\to B\leftarrow H$.}
\label{fig:cone}
\end{figure}

\begin{example}\label{ex:no-pullback}
  Let $M=(\mathbb{R}_{\ge0},+,0)$ and consider the $M$-graphs of Figure~\ref{fig:cone}: $B$ has nodes $x,x'$ and one arc $\gamma\colon x\to x'$ labelled $2$; $G$ has nodes $a,a'$ and two arcs $\alpha_1,\alpha_2\colon a\to a'$ labelled $1$; $H$ is isomorphic to $G$, but its nodes are called $b,b'$ and its arcs $\beta_1,\beta_2\colon b\to b'$, again labelled $1$. The morphisms $\phi_G\colon G\to B$ ($a\mapsto x$, $a'\mapsto x'$, $\alpha_i\mapsto\gamma$) and $\phi_H\colon H\to B$ (analogously) are $M$-fibrations. A \emph{cone} over this cospan is an $M$-graph $Q$ with two $M$-fibrations $\psi_G\colon Q\to G$ and $\psi_H\colon Q\to H$ such that $\phi_G\circ\psi_G=\phi_H\circ\psi_H$; a morphism of cones $u\colon Q\to Q'$ is itself an $M$-fibration with $\psi'_G\circ u=\psi_G$ and $\psi'_H\circ u=\psi_H$; a pullback of the cospan is a terminal cone.

  Commutation forces every node of $Q$ to lie \emph{over} $(a,b)$ or over $(a',b')$ (i.e., to be mapped by $\psi_G,\psi_H$ to $a,b$ or to $a',b'$, respectively), and every arc of $Q$ to go from a node over $(a,b)$ to a node over $(a',b')$, being mapped to some $\alpha_i$ by $\psi_G$ and to some $\beta_j$ by $\psi_H$. For a node $z$ of $Q$ over $(a',b')$ let
  \begin{multline*}
    m_{ij}(z)=\bigoplus\{\lambda(d) : d\in Q(-,z),\mlbr \psi_G(d)=\alpha_i,\ \psi_H(d)=\beta_j\},
  \end{multline*}
  the total in-weight of $z$ along arcs lying over $(\alpha_i,\beta_j)$. The $M$-fibration condition for $\psi_G$ at $z$ and $\alpha_i$ reads $m_{i1}(z)+m_{i2}(z)=\lambda(\alpha_i)=1$, and the one for $\psi_H$ at $z$ and $\beta_j$ reads $m_{1j}(z)+m_{2j}(z)=1$; hence, the matrix $m(z)=(m_{ij}(z))$ is 
  \[
    m(z)=\begin{pmatrix} t & 1-t\\ 1-t & t\end{pmatrix}\qquad\text{for some } t=t(z)\in[0,1].
  \]
  Conversely, for every $t\in[0,1]$ the $M$-graph $Q_t$ of Figure~\ref{fig:cone} (nodes $s$ over $(a,b)$ and $z$ over $(a',b')$, arcs $d_{ij}\colon s\to z$ with $\psi_G(d_{ij})=\alpha_i$, $\psi_H(d_{ij})=\beta_j$, and labels $\lambda(d_{11})=\lambda(d_{22})=t$, $\lambda(d_{12})=\lambda(d_{21})=1-t$) is a cone with $t(z)=t$: both fibration conditions hold at $z$, and there is nothing to check at $s$.

  Now let $u\colon Q\to Q'$ be a morphism of cones and $z$ a node of $Q$ over $(a',b')$. Every arc $d\in Q(-,z)$ with $\psi_G(d)=\alpha_i$ and $\psi_H(d)=\beta_j$ is mapped by $u$ to an arc $e\in Q'(-,u(z))$ with $\psi'_G(e)=\alpha_i$ and $\psi'_H(e)=\beta_j$, and the $M$-fibration condition for $u$ says that $\lambda(e)$ is the sum of the labels of the arcs of $Q(-,z)$ mapped to $e$; summing over all such $e$ we obtain $m_{ij}(u(z))=m_{ij}(z)$, that is, \emph{morphisms of cones preserve the matrix $m$}. Hence, if a pullback $P$ existed, the images in $P$ of the nodes $z$ of the cones $Q_t$, $t\in[0,1]$, would be pairwise distinct (they have distinct matrices): as a consequence, $P$ would have uncountably many nodes. No finite (not even countable) pullback exists.

  Note that in the standard case (all labels equal to $1$) unique lifting makes the fibres of each arc singletons, so the matrix $m(z)$ is $1\times 1$ and forced to be $(1)$: this is why the pullback of the underlying graphs is a pullback of standard fibrations (see~\cite[Theorem~6.1]{BoVGF}, which proves even more), and why the argument above needs labels different from $1$; the obstruction is the non-uniqueness of the matrices $m(z)$, not their existence.
\end{example}

\begin{example}\label{ex:no-universal}
  The graphs $B$ and $G$ of Figure~\ref{fig:cone} also show that universal total graphs~\cite[Sect.~3]{BoVGF} lose their universal property; note that here it suffices to take $M=(\mathbb{N},+,0)$. Recall that the universal total graph of a graph at a node $x$ is the in-tree of all paths ending in $x$, fibred onto the graph in the obvious way; its universal property~\cite[Thm.~3.1]{BoVGF} states that for every fibration $\phi\colon H\to B$ and every node $y$ of $H$ over $x$ there is exactly one fibration from the universal total graph of $B$ at $x$ to $H$ sending the root to $y$; in particular, the universal total graphs of $B$ at $x$ and of $H$ at $y$ are isomorphic. For $M$-graphs the same construction (copying the labels) still yields an $M$-fibration, but the universal property fails: $B$ and $G$ are in-trees, and each is its own universal total graph at $x'$ and at $a'$, respectively; yet there is no $M$-fibration $\psi\colon B\to G$ with $\psi(x')=a'$ (the arc $\gamma$ would have to be mapped to $\alpha_1$ or to $\alpha_2$, and the fibration condition at $a'$ for the other arc would read $0=1$), and $B\not\cong G$. What survives is the fact that the two in-trees are \emph{bisimilar} as weighted transition systems, and indeed the whole theory of Section~\ref{sec:min-base} can be seen as a theory of weighted bisimilarity of unfoldings; we shall not pursue this point of view here.
\end{example}

\section{$\epsilon$-approximate $M$-fibrations}\label{sec:approx}

We want to move on to study approximate versions of $M$-fibrations, which are relevant in applications where we tolerate that the fibration condition may hold only approximately. \chg{Threshold-based relaxations of fibration symmetry have appeared before in special settings: structurally, in the quasifibrations of~\cite{BoLMQF}; on node activities, in~\cite{VePRFS}; and, for the in-weights of layered networks, in the parallel work~\cite{VePHM}. The notion introduced below subsumes these weight-based relaxations in the general setting of labels taken from an arbitrary metric monoid.} 
We first need to have a measure of distance between elements of the monoid:

\begin{definition}
  A \emph{metric monoid} $(M,\oplus,0,d)$ is given by a commutative monoid $(M,\oplus,0)$ and a metric $d: M\times M \to \mathbb{R}_{\ge 0}$ on $M$ such that for all $x,y,z\in M$ we have $d(x\oplus z,y\oplus z)=d(x,y)$.
\end{definition}

By the triangular inequality, we have 
\begin{multline*}
  d(x\oplus x',y\oplus y')\leq d(x\oplus x',y\oplus x')+d(y\oplus x',y\oplus y')\mleq
  d(x,y)+d(x',y')
\end{multline*} 
and this extends to finite sums: this property is called \emph{sum-subadditivity}.

Now, we can measure how much a morphism of $M$-graphs fails to be an $M$-fibration:
\begin{definition}[Approximate $M$-fibration]
  Given a metric monoid $(M,\oplus,0,d)$, an $M$-graph $G$, and a morphism of $M$-graphs $\phi: G \to H$, for every arc $\bar a$ of $H$ and every node $x\in \phi^{-1}(t(\bar a))$ we define the \emph{local fibration error} of $\phi$ at $(\bar a,x)$ as
  \[
    \delta_\phi(\bar a,x)=d\Bigl(\bigoplus_{a\in \phi^{-1}(\bar a)\cap G(-,x)} \lambda(a),\, \lambda(\bar a)\Bigr).
  \]
  The \emph{global fibration error} of $\phi$ is then defined as
  \[
    \Delta(\phi)=\max_{x \in N_G} 
      \sum_{\bar a \in t^{-1}(\varphi(x))}\delta_\phi(\bar a,x).
  \]
  A morphism $\phi$ is an \emph{$\epsilon$-approximate $M$-fibration} if $\Delta(\phi)\leq \epsilon$.
\end{definition}

It is immediate to observe that $\Delta(\phi)=0$ if and only if $\phi$ is an $M$-fibration (all local errors are zero). Moreover, global fibration error is subadditive with respect to composition of morphisms: 

\begin{proposition}
  If $\phi: G \to H$ and $\psi: H \to K$ are morphisms of $M$-graphs, then $\Delta(\psi\circ\phi)\leq \Delta(\phi)+\Delta(\psi)$.
\end{proposition}

We can also define how far is a pair of graphs from being linked by an epimorphic $M$-fibration:

\begin{definition}[Fibration distance]
  We define the \emph{fibration distance} between $G$ and $H$ as
  \[
    \fd(G,H)=\inf\{\Delta(\phi) : \phi: G \to H \text{ is  epimorphic}\}.
  \]
  Moreover, we define the function $\beta_G: \RR_{\ge 0} \to \NN$ as
  \[
    \beta_G(\epsilon)=\min\{|N_H| : \fd(G,H)\leq \epsilon\}.
  \]
\end{definition}
The $\fd$ function is a quasi-metric in the sense of Lawvere: $\fd(G,G)=0$ and  $\fd(G,B)=0$ if and only if $G$ can be epimorphically fibred over $B$. Furthermore, $\fd(G,B)\leq \fd(G,H)+\fd(H,B)$. But $\fd$ is \emph{not} symmetric, hence it is not a metric.

The function $\beta_G$, in mundane words, says how much you can compress $G$ with an approximate fibration if you allow for an error of at most $\epsilon$. It is non-increasing (if you allow for more error, you can compress more), and $\beta_G(0)=|N_{\hat G}|$ (the number of nodes of the minimum base of $G$). Moreover, for sufficiently large $\epsilon$, $\beta_G(\epsilon)=1$ (the trivial graph with one node and one loop labelled with 0 is an $\epsilon$-approximate fibration of any graph, for sufficiently large $\epsilon$: you just take $\epsilon$ to be the maximum sum of incoming labels over all the nodes of $G$).

The function $\beta_G$ expresses the trade-off between compression and error, and it is a natural object of study in applications. 

\section{$\epsilon$-approximate $M$-equitable partitions}\label{sec:approx-equitable}

Like in the exact case, we can draw a parallel between $\epsilon$-approximate $M$-fibrations and $\epsilon$-approximate $M$-equitable partitions. To define them, let us first observe that in a metric monoid $(M,\oplus,0,d)$ we can define a distance between vectors of elements of $M$ of the same length: for ${\mathbf x}=(x_1,\dots,x_k),{\mathbf y}=(y_1,\dots,y_k)\in M^k$, we define $d({\mathbf x},{\mathbf y})=\sum_{i=1}^k d(x_i,y_i)$. This is a metric on $M^k$ and it is sum-subadditive with respect to the componentwise sum of vectors. We use the same notation $d$ for this metric on vectors, and we write ${\mathbf 0}$ for the vector of length $k$ whose components are all $0$.

\begin{definition}[Approximate $M$-equitable partition]
  Given a metric monoid $(M,\oplus,0,d)$, an $M$-graph $G$, and a partition $\Pi$ of $N_G$ with $p$ parts, define the vector of \emph{local in-weights} of a node $x$ as
  \[
    {\mathbf w}_x=\bigl(W(C,x)\bigr)_{C\in \Pi} \in M^p.
  \]
  A vector ${\mathbf c} \in M^p$ is a \emph{center of radius $r$ for the part $D$} iff $d({\mathbf w}_x,{\mathbf c})\leq r$ for all $x\in D$. The \emph{inequity} of $\Pi$, $\upsilon(\Pi)$, is the minimum value $r$ such that there exists a center of radius $r$ for every part of $\Pi$.
  We say that $\Pi$ is an \emph{$\epsilon$-approximate $M$-equitable partition} if $\upsilon(\Pi)\leq \epsilon$.
\end{definition}

Note that $\Pi$ is a $0$-approximate $M$-equitable partition if and only if there are centers of radius $0$ for all parts. The existence of a center of radius $0$ for a part $D$ is equivalent to the existence of a vector ${\mathbf c}$ such that ${\mathbf w}_x={\mathbf c}$ for all $x\in D$, which precisely means that ${\mathbf w}_x={\mathbf w}_y$ for all $x,y\in D$, i.e., that $\Pi$ is an exact $M$-equitable partition. Hence, the notion of $\epsilon$-approximate $M$-equitable partition is indeed a generalization of the notion of $M$-equitable partition.

There is an obvious parallel between $\epsilon$-approximate $M$-fibrations and $\epsilon$-approximate $M$-equitable partitions (the analogue of Proposition~\ref{prop:fib-equi}); we state it without proof:
\begin{proposition}
  \label{prop:afib-equi}
  \begin{enumerate}
    \item \label{enu:afib-equi} Let $\phi: G \to H$ be an epimorphic $\epsilon$-approximate $M$-fibration: then the partition $\Pi_\sim$ associated with the equivalence relation $\sim$ defined by $x\sim y$ if and only if $\phi(x)=\phi(y)$ is an $\epsilon$-approximate $M$-equitable partition of $G$.
    \item \label{enu:aequi-fib} Let $\Pi$ be an $\epsilon$-approximate $M$-equitable partition of $G$, and let $G/\Pi$ be the simple $M$-graph obtained as the quotient of $G$ by $\Pi$ (its nodes are the parts of $\Pi$, and there is an arc $(C,D)$, labelled by $W(C,D)$, if and only if $G(C,D)\neq\emptyset$); let $\pi: G \to G/\Pi$ be the canonical projection to the quotient graph then $\pi$ is an epimorphic $\epsilon$-approximate $M$-fibration.
  \end{enumerate}
\end{proposition}

It is useful to observe that computing $\beta_G(\epsilon)$ is equivalent to computing an $\epsilon$-approximate $M$-equitable partition of $G$ with the minimum number of classes. Observe that for $\epsilon>0$, there is \emph{no coarsest $\epsilon$-approximate $M$-equitable partition}, as the following counterexample shows.

\begin{figure}
    \centering
    \begin{tikzpicture}[
      n/.style={circle, draw, fill=white, minimum size=6pt, inner sep=3pt},
      >={Stealth[length=5pt]}
    ]
    \node[n] (u)    at (0,2.2)     {$u$};
    \node[n] (x)    at (-1.7,0.6)     {$x$};
    \node[n] (y) 	at (0,0.6)     {$y$};
    \node[n] (z)    at (1.7,0.6) {$z$};

	\draw[->] (u) to[out=120, in=60, looseness=4, min distance=8mm] node[above] {10} (u);
    \draw[->] (u) -- (x) node[midway,left] {1};
    \draw[->] (u) -- (y) node[midway,left] {1.2};
    \draw[->] (u) -- (z) node[midway,right] {1.4};
  \end{tikzpicture}
    \caption{An example showing that there is no coarsest $\epsilon$-approximate $M$-equitable partition.}
    \label{fig:no-coarsest}
  \end{figure}
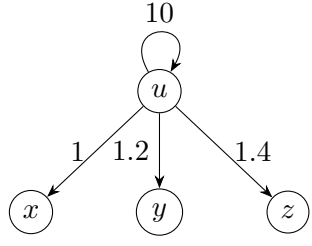

\begin{example}
  Consider the graph $G$ in Figure~\ref{fig:no-coarsest}, with the monoid $(\RR,+,0)$ and the Euclidean distance. Clearly $\beta_G(0)=4$ because $G$ itself is fibration prime (all input weights are distinct). Let us find the inequity of all non-trivial partitions leaving $u$ alone: $\upsilon(\{u | xy | z\})=0.1$, $\upsilon(\{u | x | yz\})=0.1$, $\upsilon(\{u | xz | y\})=0.2$, $\upsilon(\{u | xyz\})=0.2$; mixing $u$ with other nodes produces an inequity of $4.5$. Hence, for $0.1\leq\epsilon<0.2$, there are two $\epsilon$-approximate $M$-equitable partitions of $G$ producing non-isomorphic quotient $M$-graphs (the labels on the arcs will be different, and node $y$ will be mapped to different nodes of the quotient graphs).
  The actual values of $\beta_G(\epsilon)$ are: 
  \[
    \beta_G(\epsilon)=\begin{cases}
      4 & \text{if } \epsilon \in [0,0.1)\\
      3 & \text{if } \epsilon \in [0.1,0.2)\\
      2 & \text{if } \epsilon \in [0.2,4.5)\\
      1 & \text{if } \epsilon\geq 4.5.
    \end{cases}
  \] 
\end{example}

Finding an $\epsilon$-approximate $M$-equitable partition of $G$ with the minimum number of classes is a difficult problem:
it is an optimization problem, which can be formulated as a MILP, but we conjecture it to be NP-hard. We leave the study of this conjecture to future work.

We can overcome this difficulty by considering a greedy algorithm that iteratively refines the partition of nodes until it stabilizes, as in the exact case. The difference is that now we have to check whether two nodes are $\epsilon$-approximately equivalent. This is what Algorithm~\ref{alg:tolerant} does: it computes an $\epsilon$-approximate $M$-equitable partition of $G$, and returns the centers of the parts. Centers can be computed in different ways, depending on the monoid we are working with.
One approach that works regardless of the monoid is to take \emph{any} vector as center, or to take the Chebyshev center (the vector that minimizes the maximum distance from the vectors of the part); the latter solution makes the computation quadratic. 

The algorithm is greedy, and it may not return a partition with the minimum number of classes, but it is guaranteed to return a partition with radius at most $\epsilon$ for each class. The algorithm is a generalization of the one in~\cite{ValFSO}, which is the case $M=(\mathbb{R}_{\ge 0},+,0)$ and $\mathrm{ctr}$ is the mean.  

The nodes and arcs of the $\epsilon$-approximate minimum base are obtained by the $\epsilon$-approximate equitable partition coming from Algorithm~\ref{alg:tolerant} (parts are the nodes, an arc is present from $C$ to $D$ if and only if $W(C,D) \neq 0$), like the exact case. For the labels, though, we use the centers of the parts: the label of the arc $(C,D)$ is the $C$-component of the center of $D$. 

The $\epsilon$-approximate minimum base is not unique, because the centers of the parts are not unique. Once the partition has been computed, though, the number of nodes of the resulting base does not depend on the choice of the centers; the partition itself, however, may depend on the center rule adopted (see Example~\ref{ex:algo-H}). 

\begin{algorithm}[t]
\caption{Certified $\epsilon$-equitable partition (tolerant refinement, \texttt{epsilon\_partition})}
\label{alg:tolerant}
  \begin{algorithmic}[1]
  \Require $M$-graph $G$ over a metric monoid $(M,\oplus,0,d)$; $\epsilon\ge0$; a center rule $\mathrm{ctr}$ mapping a finite family of vectors of $M^p$ to a vector of $M^p$ (e.g., the coordinatewise mean, the $\ell^1$ Chebyshev center, or anyone of the vectors)
  \State $\Pi\gets\{N_G\}$ \Comment{trivial partition}
  \Repeat 
    \State for every node $x$: ${\mathbf w}_x\gets\bigl(W(C,x)\bigr)_{C\in\Pi}$
    \State $\Pi'\gets\emptyset$
    \For{each class $D\in\Pi$}
      \State $R\gets D$ \label{ln:tol-begin}
      \While{$R\neq\emptyset$}
        \State $s\gets$ a seed in $R$ (the first time: any node; afterwards: the one farthest from the previous seed)
        \State $A\gets\{x\in R: d({\mathbf w}_x,{\mathbf w}_s)\leq 2\epsilon\}$ \Comment{candidates}
        \State ${\mathbf c}\gets\mathrm{ctr}(\{{\mathbf w}_x : x\in A\})$
        \While{$\max_{x\in A}d({\mathbf w}_x,{\mathbf c})>\epsilon$}
          \State remove from $A$ the node $x\neq s$ farthest from ${\mathbf c}$
          \State ${\mathbf c}\gets\mathrm{ctr}(\{{\mathbf w}_x : x\in A\})$
        \EndWhile
        \State $A\gets A\cup\{x\in R: d({\mathbf w}_x,{\mathbf c})\leq \epsilon\}$ \Comment{absorb}
        \State add the group $A$, with center ${\mathbf c}$, to $\Pi'$
        \State $R\gets R\setminus A$ \label{ln:tol-end}
      \EndWhile
    \EndFor
    \If{$|\Pi'|=|\Pi|$} {\bf break} \EndIf
    \State $\Pi\gets\Pi'$
  \Until{false}
  \State \Return{$\Pi$} with its centers \Comment{every class has radius $\le\epsilon$ w.r.t.\ its own aggregated vectors}
\end{algorithmic}
\end{algorithm}

\begin{example}\label{ex:algo-H}
  Let us try to use Algorithm~\ref{alg:tolerant} to compute an $\epsilon$-approximate $M$-equitable partition of the $M$-graph $H$ of Figure~\ref{fig:example}, with $M=(\RR,+,0)$ and $\epsilon=1$. We use the coordinatewise mean as center rule.
  
  We start with the trivial partition $\Pi=\{\{0,1,2,3,4\}\}$ and we compute the local in-weights of the nodes: ${\mathbf w}_0=(16)$, ${\mathbf w}_1=(14)$, ${\mathbf w}_2=(38)$, ${\mathbf w}_3=(30)$, ${\mathbf w}_4=(30)$. The first seed is $0$, and the candidate set is $\{0,1\}$; the center is computed as ${\mathbf c}=(15)$ (the average between $(14)$ and $(16)$). No node is added to the candidate set. The other class that is created in the split is $\{3,4\}$ (they literally share the same vector). So we have a new partition $\Pi'=\{\{0,1\},\{2\},\{3,4\}\}$, which is different from the previous one, and we continue. The new vectors are ${\mathbf w}_0=(6,0,10)$, ${\mathbf w}_1=(5,0,9)$, ${\mathbf w}_2=(38,0,0)$, ${\mathbf w}_3=(0,30,0)$, ${\mathbf w}_4=(0,30,0)$. The first seed is $0$, and the candidate set is $\{0,1\}$; the center computed this time is ${\mathbf c}=(5.5,0,9.5)$ (the average between $(6,0,10)$ and $(5,0,9)$). Proceeding this way, we will in fact not split any further, so the algorithm terminates with the $\epsilon$-approximate $M$-equitable partition $\Pi=\{\{0,1\},\{2\},\{3,4\}\}$, with centers ${\mathbf c}_{01}=(5.5,0,9.5)$, ${\mathbf c}_2=(38,0,0)$, and ${\mathbf c}_{34}=(0,30,0)$. The resulting base is shown in Figure~\ref{fig:minimum-approx-base}: the nodes are the parts of the partition, and the arc from part $C$ to part $D$ (present iff $H$ has an arc from $C$ to $D$) is labelled by the $C$-th component of the center of $D$.

  An exhaustive search shows that the algorithm (when using coordinatewise mean) produces a base with 4 nodes when $\epsilon \in [0,1)$, 3 nodes when $\epsilon \in [1,13.5)$, and 1 node when $\epsilon \geq 13.5$. The actual function $\beta_H(\epsilon)$ is 4 in $[0,1)$, 3 in $[1,12)$ and 1 when $\epsilon \geq 12$. This comparison shows that the algorithm is not optimal in the interval $[12,13.5)$, but it is optimal in the other two intervals; using the exact $\ell^1$ Chebyshev center as center rule, instead, the algorithm attains $\beta_H(\epsilon)$ for \emph{every} $\epsilon$. Note that $\beta_H$ never takes the value $2$: the best partition with two parts, $\{0\},\{1,2,3,4\}$, has inequity $13$, \emph{larger} than the inequity $12$ of the trivial partition --- the smallest feasible $\epsilon$ is not monotone in the number of parts.
\end{example}

\begin{figure}
  \centering
    \resizebox{0.1\textwidth}{!}{
    \begin{tikzpicture}[
      n/.style={circle, draw, fill=white, minimum size=6pt, inner sep=0pt},
      >={Stealth[length=5pt]}
    ]
    \node[n] (top)    at (0,2.2)     {$\{0,1\}$};
    \node[n] (center) at (0,0.6)     {$\{2\}$};
    \node[n] (b)      at (0,-1.3) {$\{3,4\}$};

    \draw[->] (top) -- (center) node[midway,left] {38};
    \draw[->] (center) -- (b) node[midway,left] {30};
    \draw[->] (b) to[bend right=45] node[pos=0.5, above, xshift=8pt] {9.5} (top);
	\draw[->] (top) to[out=120, in=60, looseness=4, min distance=8mm] node[above] {5.5} (top);
  \end{tikzpicture} }
  \caption{The $\epsilon$-approximate base of the graph $H$ in Figure~\ref{fig:example} produced by Algorithm~\ref{alg:tolerant}, when the monoid is $(\RR,+,0)$, $\epsilon=1$ and we are using the coordinatewise mean as center rule (here it is in fact a smallest $\epsilon$-approximate base: no $\epsilon$-equitable partition with two parts exists). Observe that if we were using the monoid $(\ZZ,+,0)$ instead, we could not apply the coordinatewise mean as center rule (division is not defined in $\ZZ$); using the $\ell^1$ Chebyshev center with integer coordinates, instead, we would get ${\mathbf c}_{01}=(5,0,10)$ or $(6,0,9)$ (both at distance $1$ from ${\mathbf w}_0$ and ${\mathbf w}_1$), whereas ${\mathbf c}_2=(38,0,0)$ and ${\mathbf c}_{34}=(0,30,0)$ would be unchanged.}
  \label{fig:minimum-approx-base}
\end{figure}
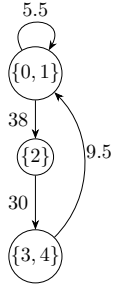

\section{Application: compression of neural networks}\label{sec:application}

Fibrations embed a notion of symmetry in a graph: two nodes that can be fibred together, receive the same input from the same (or equivalent) sources. This notion of symmetry can be fruitfully applied to compress neural networks: two neurons that are equivalent in this sense can be merged into one, and the weights of the incoming arcs can be summed. This idea was introduced in~\cite{VePRFS}: here we show the mathematical foundations of this approach. \chg{While the present paper was being finalized, the approach of~\cite{VePRFS} was further developed in~\cite{VePHM}, which introduces $\epsilon$-fibers of weighted layered networks together with a mean-weight compression rule, and studies the emergence of these symmetries during training; the framework presented here provides a general algebraic and metric foundation for that line of work.} The fact that we are using general monoids allows us to treat different types of neural networks, including convolutional neural networks.

Let us briefly describe how a neural network can be represented as an $M$-graph, using the well-known MNIST classification dataset as reference~\cite{LeBBH98}.

In a simple multilayer perceptron (MLP), the translation is quite simple: the nodes of the graph are the neurons, and the arcs are the synaptic connections between them. The weights on the arcs are the synaptic weights, so the natural monoid is $(\RR,+,0)$. In fact, it is convenient to add explicitly a bias node and to introduce a normalization factor that is different at each layer, so that the distances between the in-weights of the neurons are comparable across layers. Details are given in the caption of Figure~\ref{fig:mlpgraph}.

\begin{figure}
  \centering  
  \resizebox{\netgraphwidth}{!}{
\begin{tikzpicture}[>=Stealth, font=\small]
\node[font=\bfseries] at (7.5,3.4) {(a) The MLP network: four fully-connected layers, every unit carries one number};
\foreach \x/\n/\col/\lab in {1.5/784/gray/{input\\ $784$ pixels}, 5.5/300/blue/{hidden 1\\ $300$ units}, 9.5/100/red/{hidden 2\\ $100$ units}, 13.5/10/violet/{output\\ $10$ units}} {
  \draw[fill=\col!25,draw=\col!70!black,thick,rounded corners=2pt] (\x-0.35,0.2) rectangle ++(0.7,2.4);
  \foreach \yy in {0.5,0.9,1.3,1.7,2.1,2.5} \fill[\col!70!black] (\x,\yy) circle (1.4pt);
  \node[align=center,font=\footnotesize] at (\x,-0.35) {\lab};
}
\draw[->,thick] (2.0,1.4) -- node[above,font=\footnotesize,align=center] {fully connected\\ $784{\times}300$ weights} (5.0,1.4);
\draw[->,thick] (6.0,1.4) -- node[above,font=\footnotesize,align=center] {fully connected\\ $300{\times}100$ weights} (9.0,1.4);
\draw[->,thick] (10.0,1.4) -- node[above,font=\footnotesize,align=center] {fully connected\\ $100{\times}10$ weights} (13.0,1.4);
\node[font=\footnotesize,align=center,text=blue!70!black] at (5.5,-1.35) {unit $j$ computes\\ $\mathrm{ReLU}\bigl(b_j+\sum_i W_{ji}\,x_i\bigr)$};

\begin{scope}[yshift=-6.0cm]
\node[font=\bfseries] at (7.5,2.4) {(b) The $M$-graph: one node per unit plus a bias node, one arc per nonzero weight.};
\tikzset{nd/.style={circle,draw,thick,minimum size=5mm,inner sep=0pt,font=\scriptsize}}
\foreach \i/\y in {1/1.2,2/0.4,3/-0.4} \node[nd,fill=gray!30] (i\i) at (1.5,\y) {$x_{\i}$};
\node at (1.5,-1.05) {$\vdots$}; \node[nd,fill=gray!30] (in) at (1.5,-1.8) {$x_{784}$};
\foreach \i/\y in {1/1.2,2/0.4,3/-0.4} \node[nd,fill=blue!25] (h\i) at (5.5,\y) {$h_{\i}$};
\node at (5.5,-1.05) {$\vdots$}; \node[nd,fill=blue!25] (hn) at (5.5,-1.8) {$h_{300}$};
\foreach \i/\y in {1/1.2,2/0.4,3/-0.4} \node[nd,fill=red!25] (k\i) at (9.5,\y) {$k_{\i}$};
\node at (9.5,-1.05) {$\vdots$}; \node[nd,fill=red!25] (kn) at (9.5,-1.8) {$k_{100}$};
\foreach \i/\y in {1/0.9,2/0.1} \node[nd,fill=violet!25] (o\i) at (13.5,\y) {$o_{\i}$};
\node at (13.5,-0.65) {$\vdots$}; \node[nd,fill=violet!25] (on) at (13.5,-1.4) {$o_{10}$};
\node[nd,fill=yellow!40,minimum size=6mm] (bias) at (7.5,-3.4) {bias};
\foreach \s in {i1,i2,i3,in} \foreach \t in {h1,h2,h3,hn} \draw[->,blue!60!black,thin,opacity=0.5] (\s) -- (\t);
\foreach \s in {h1,h2,h3,hn} \foreach \t in {k1,k2,k3,kn} \draw[->,red!60!black,thin,opacity=0.5] (\s) -- (\t);
\foreach \s in {k1,k2,k3,kn} \foreach \t in {o1,o2,on} \draw[->,violet!70!black,thin,opacity=0.5] (\s) -- (\t);
\foreach \t in {hn,kn,on} \draw[->,yellow!60!black,dashed] (bias) -- (\t);
\node[align=center,font=\footnotesize,blue!70!black] at (3.5,1.9) {$784{\times}300$ arcs\\ label $=W_{ji}\in\mathbb{R}$};
\node[align=center,font=\footnotesize,red!60!black] at (7.5,1.9) {$300{\times}100$ arcs\\ label $\in\mathbb{R}$};
\node[align=center,font=\footnotesize,violet!70!black] at (11.5,1.9) {$100{\times}10$ arcs\\ label $\in\mathbb{R}$};
\node[align=left,font=\footnotesize,yellow!50!black,anchor=west] at (8.1,-3.4) {bias arcs to every hidden/output unit, label $=b\in\mathbb{R}$};
\begin{scope}[on background layer]
 \node[rounded corners,draw=gray!60,dashed,fit=(i1)(in),inner sep=4pt,label={[gray,font=\scriptsize]below:frozen singletons}] {};
 \node[rounded corners,draw=blue!50,dashed,fit=(h1)(hn),inner sep=4pt,label={[blue!60,font=\scriptsize]below:initial class ``layer 1''}] {};
 \node[rounded corners,draw=red!50,dashed,fit=(k1)(kn),inner sep=4pt,label={[red!60,font=\scriptsize]below:``layer 2''}] {};
\end{scope}
\node[font=\scriptsize,gray] at (13.5,-2.05) {frozen singletons};
\end{scope}
\end{tikzpicture}}
  \caption{LeNet-300-100~\cite{LeBBH98}, an MLP neural network for the MNIST classification problem. Above, the usual representation: here, layers are fully connected. When represented as an $M$-graph, weights live on the arcs (we added a bias node, to spread biases to the layers), and input and output layers are frozen (their units cannot be merged during the process). Instead of working on $(\RR,+,0)$, it is more convenient to use $(\RR\oplus\RR\oplus\RR,+,0,d_1\oplus d_2 \oplus d_3)$ so that distances are measured differently on each layer: explicitly, we can define $d_i(u,v)=|u-v|/s_i$ where $s_i$ is the median distance between the total in-weights of the nodes of layer $i$: this way, we can normalize distances across layers when the actual weights vary significantly. Merging a class is the quotient construction: in-weights $\leftarrow$ center, out-weights $\leftarrow$ fibre sums.}
\label{fig:mlpgraph}
\end{figure}

In a convolutional neural network (CNN), it is convenient to represent the network as a graph where the nodes are the channels of the network, and the arcs are the convolutional kernels that connect them. The weights on the arcs are kernels of appropriate sizes. A detailed explanation is given in the caption of Figure~\ref{fig:cnngraph}.

\begin{figure}
  \centering 
  \resizebox{\netgraphwidth}{!}{
\begin{tikzpicture}[>=Stealth, font=\small]

\newcommand{\stack}[6]{
  \foreach \i in {0,...,#3} {
    \draw[fill=#5!30,draw=#5!70!black,thick] ({#1+0.12*\i},{#2-0.12*\i}) rectangle ++(#4,#4);
  }
  \node[align=center,font=\footnotesize] at ({#1+#4/2+0.06*#3},{#2-0.12*#3-0.35}) {#6};
}
\node[font=\bfseries] at (7.5,3.6) {(a) The CNN network: a pile of same-sized arrays at every stage};
\stack{0}{1.6}{0}{1.6}{gray}{input\\ $1\times28\times28$}
\stack{3.2}{1.5}{4}{1.35}{blue}{conv1 (16 maps)\\ $16\times24\times24$}
\stack{6.4}{1.55}{4}{0.75}{blue}{pool\\ $16\times12\times12$}
\stack{9.2}{1.6}{7}{0.5}{red}{conv2 (32 maps)\\ $32\times8\times8$}
\stack{12.4}{1.7}{7}{0.28}{red}{pool\\ $32\times4\times4$}
\draw[->,thick] (1.75,2.2) -- node[above,font=\footnotesize,align=center] {conv $5{\times}5$\\ 16 kernels} (3.1,2.2);
\draw[->,thick] (5.2,2.2) -- node[above,font=\footnotesize] {max-pool} (6.3,2.2);
\draw[->,thick] (7.8,2.2) -- node[above,font=\footnotesize,align=center] {conv $5{\times}5$\\ $16{\times}32=512$ kernels} (9.1,2.2);
\draw[->,thick] (10.8,2.2) -- node[above,font=\footnotesize] {max-pool} (12.3,2.2);
\node[font=\footnotesize,align=center,text=blue!70!black] at (3.9,-1.0) {same kernel $K_{f,c}$ slid over\\ all positions of a map};
\draw[->,blue!70!black,thin] (3.9,-0.6) -- (4.1,0.55);
\draw[->,thick] (13.0,0.15) -- node[right,font=\footnotesize,align=left] {all entries} (13.0,-0.55);
\draw[fill=green!25,draw=green!60!black,thick] (12.3,-1.05) rectangle ++(1.4,0.5);
\node[font=\footnotesize,anchor=east] at (12.1,-0.8) {flatten $512$};
\draw[->,thick] (13.0,-1.1) -- node[right,font=\footnotesize,align=left] {fc $512{\times}128$} (13.0,-1.85);
\draw[fill=orange!30,draw=orange!70!black,thick] (12.55,-2.35) rectangle ++(0.9,0.5);
\node[font=\footnotesize,anchor=east] at (12.1,-2.1) {fc1 $128$};
\draw[->,thick] (13.0,-2.4) -- node[right,font=\footnotesize,align=left] {fc $128{\times}10$} (13.0,-3.15);
\draw[fill=violet!30,draw=violet!70!black,thick] (12.75,-3.65) rectangle ++(0.5,0.5);
\node[font=\footnotesize,anchor=east] at (12.1,-3.4) {out $10$};

\begin{scope}[yshift=-8.2cm]
\node[font=\bfseries] at (7.5,2.7) {(b) The $M$-graph: one node per channel/unit (a sheet of the pile), one arc per kernel or weight};
\tikzset{nd/.style={circle,draw,thick,minimum size=5mm,inner sep=0pt,font=\scriptsize}}
\node[nd,fill=gray!30] (in) at (0.8,0) {in};
\foreach \i/\y in {1/1.2,2/0.4,3/-0.4} \node[nd,fill=blue!25] (c1\i) at (4.0,\y) {$c_{\i}$};
\node at (4.0,-1.05) {$\vdots$}; \node[nd,fill=blue!25] (c1n) at (4.0,-1.8) {$c_{16}$};
\foreach \i/\y in {1/1.2,2/0.4,3/-0.4} \node[nd,fill=red!25] (c2\i) at (7.2,\y) {$d_{\i}$};
\node at (7.2,-1.05) {$\vdots$}; \node[nd,fill=red!25] (c2n) at (7.2,-1.8) {$d_{32}$};
\foreach \i/\y in {1/1.2,2/0.4,3/-0.4} \node[nd,fill=orange!30] (u\i) at (10.4,\y) {$u_{\i}$};
\node at (10.4,-1.05) {$\vdots$}; \node[nd,fill=orange!30] (un) at (10.4,-1.8) {$u_{128}$};
\foreach \i/\y in {1/0.9,2/0.1} \node[nd,fill=violet!25] (o\i) at (13.6,\y) {$o_{\i}$};
\node at (13.6,-0.65) {$\vdots$}; \node[nd,fill=violet!25] (on) at (13.6,-1.4) {$o_{10}$};
\node[nd,fill=yellow!40,minimum size=6mm] (bias) at (7.2,-3.4) {bias};
\foreach \t in {c11,c12,c13,c1n} \draw[->,blue!70!black] (in) -- (\t);
\foreach \s in {c11,c12,c13,c1n} \foreach \t in {c21,c22,c23,c2n} \draw[->,red!60!black,thin,opacity=0.5] (\s) -- (\t);
\foreach \s in {c21,c22,c23,c2n} \foreach \t in {u1,u2,u3,un} \draw[->,orange!80!black,thin,opacity=0.5] (\s) -- (\t);
\foreach \s in {u1,u2,u3,un} \foreach \t in {o1,o2,on} \draw[->,violet!70!black,thin,opacity=0.5] (\s) -- (\t);
\foreach \t in {c1n,c2n,un,on} \draw[->,yellow!60!black,dashed] (bias) -- (\t);
\node[align=center,font=\footnotesize,blue!70!black] at (2.4,2.0) {16 arcs\\ label $=K_{f,\mathrm{in}}\in\mathbb{R}^{5\times5}$};
\node[align=center,font=\footnotesize,red!60!black] at (5.6,2.0) {$16{\times}32=512$ arcs\\ label $=K_{d,c}\in\mathbb{R}^{5\times5}$};
\node[align=center,font=\footnotesize,orange!80!black] at (8.8,2.05) {$32{\times}128=4096$ arcs\\ label $\in\mathbb{R}^{16}$\\ ($4{\times}4$ positions of $d$ read by $u$)};
\node[align=center,font=\footnotesize,violet!70!black] at (12.0,2.0) {$128{\times}10=1280$ arcs\\ label $\in\mathbb{R}$};
\node[align=left,font=\footnotesize,yellow!50!black,anchor=west] at (7.8,-3.4) {bias arcs to every hidden/output node, label $=b\in\mathbb{R}$};
\begin{scope}[on background layer]
 \node[rounded corners,draw=blue!50,dashed,fit=(c11)(c1n),inner sep=4pt,label={[blue!60,font=\scriptsize]below:initial class ``layer 1''}] {};
 \node[rounded corners,draw=red!50,dashed,fit=(c21)(c2n),inner sep=4pt,label={[red!60,font=\scriptsize]below:``layer 2''}] {};
 \node[rounded corners,draw=orange!60,dashed,fit=(u1)(un),inner sep=4pt,label={[orange!70!black,font=\scriptsize]below:``layer 3''}] {};
\end{scope}
\node[font=\scriptsize,gray] at (0.8,-0.75) {frozen};
\node[font=\scriptsize,gray] at (13.6,-2.05) {frozen};
\end{scope}
\end{tikzpicture}}
  \caption{LeNet-5-style~\cite{LeBBH98} CNN for the MNIST classification problem. Above, the usual representation.
  Here $M=(\mathbb{R}^{5\times5},+)\oplus(\mathbb{R}^{16},+)\oplus(\mathbb{R},+)$, using $\ell^1$-distance on each component (with per-layer scale, like in the case of the MLP).
  Pooling is not in the graph: it acts inside a channel (spatially), and the pixel-level graph fibres exactly onto this one.}
\label{fig:cnngraph}
\end{figure}
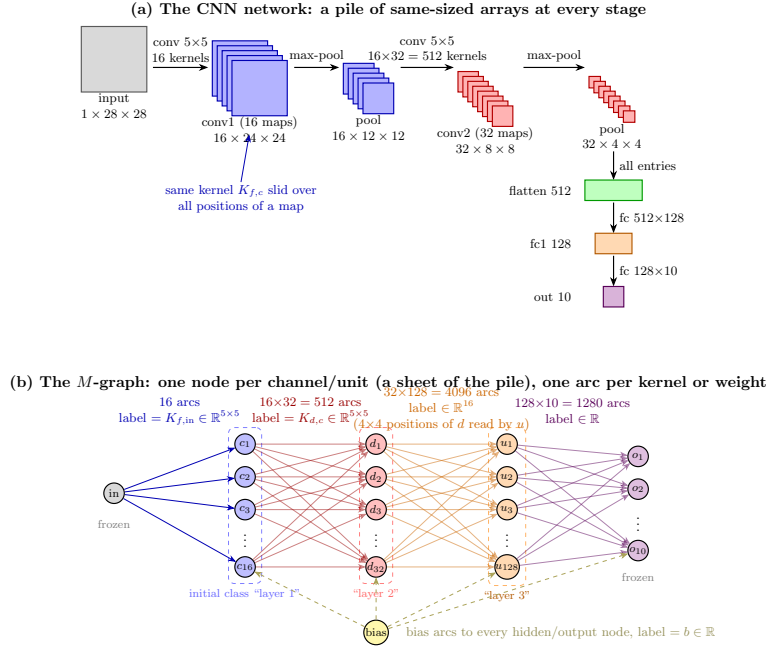

\section{Experimental results}\label{sec:exp}

\begin{figure*}[tb]
  \centering
\begin{tikzpicture}
\begin{groupplot}[
  group style={group size=3 by 1, horizontal sep=1.1cm},
  width=0.35\textwidth, height=0.21\textwidth,
  xlabel={$\epsilon$},
  ymin=0, ymax=1.05, ytick={0,0.2,...,1},
  grid=major, grid style={gray!25},
  tick label style={font=\scriptsize}, label style={font=\scriptsize},
  title style={font=\scriptsize, yshift=-1ex},
  legend style={font=\tiny, at={(0.03,0.03)}, anchor=south west, draw=gray!50, fill=white, fill opacity=0.85, text opacity=1},
  every axis plot/.append style={line width=0.8pt, mark size=1.6pt},
]
\nextgroupplot[title={MNIST MLP 784--300--100--10}, xmin=0, xmax=1.0, ylabel={compression and accuracy}]
  \fill[orange!20] (axis cs:0.35,0) rectangle (axis cs:0.5,1.05);
  \addplot[blue!70!black, mark=*] table[x=eps, y=unitfrac, col sep=comma] {fig/data/frontier_mlp.csv};
  \addplot[red!70!black, mark=square*] table[x=eps, y=acc, col sep=comma] {fig/data/frontier_mlp.csv};
  \addplot[black, densely dotted, no marks, forget plot] coordinates {(0,0.9776) (1.0,0.9776)};
  \legend{surviving units (fraction), test accuracy}
\nextgroupplot[title={MNIST CNN 16c5--32c5--128fc}, xmin=0, xmax=1.0]
  \fill[orange!20] (axis cs:0.35,0) rectangle (axis cs:0.5,1.05);
  \addplot[blue!70!black, mark=*] table[x=eps, y=unitfrac, col sep=comma] {fig/data/frontier_cnn.csv};
  \addplot[red!70!black, mark=square*] table[x=eps, y=acc, col sep=comma] {fig/data/frontier_cnn.csv};
  \addplot[black, densely dotted, no marks, forget plot] coordinates {(0,0.9909) (1.0,0.9909)};
\nextgroupplot[title={CIFAR-10 VGG16-BN}, xmin=0.28, xmax=0.62]
  \fill[orange!20] (axis cs:0.35,0) rectangle (axis cs:0.5,1.05);
  \addplot[blue!70!black, mark=*] table[x=eps, y=unitfrac, col sep=comma] {fig/data/frontier_vgg.csv};
  \addplot[red!70!black, mark=square*] table[x=eps, y=acc, col sep=comma] {fig/data/frontier_vgg.csv};
  \addplot[black, densely dotted, no marks, forget plot] coordinates {(0.28,0.9416) (0.62,0.9416)};
\end{groupplot}
\end{tikzpicture}
  \caption{Fraction of surviving units (circles) and test accuracy (squares) of the compressed networks as a function of $\epsilon$, for the three networks of Table~\ref{tab:compression-results}; the dotted line is the accuracy of the uncompressed network, and the shaded band $0.35\le\epsilon\le0.5$ contains the phase transition.}
  \label{fig:compression-results}
\end{figure*}
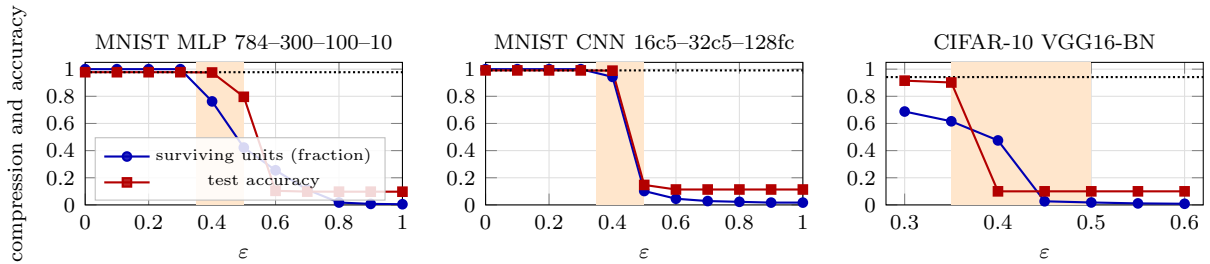

For our experiments of neural network compression, we used three trained networks: two for the MNIST dataset (the LeNet-300-100 MLP and LeNet-5 CNN architectures) and one for the CIFAR-10 dataset~\cite{KriLML}. The compression was done post-training, so it is not comparable to the method described in~\cite{VePRFS}. This time, we simply converted the trained networks into $M$-graphs, and then we applied our $\epsilon$-approximate $M$-fibration algorithm (Algorithm~\ref{alg:tolerant}) to compress them, and then finally re-converted the compressed graphs back into neural networks, and tried them on the test set to measure their performance. The results are summarized in Table~\ref{tab:compression-results}, and represented in Figure~\ref{fig:compression-results}. 

For these networks, we observed a phase transition in the accuracy of the compressed networks: for small values of $\epsilon$, the accuracy remains almost unchanged, and the size (in number of units and parameters) decreases, but when $\epsilon$ reaches a critical value (between $0.35$ and $0.5$ in all three cases), the accuracy drops dramatically. 

The sharp transition is a property of the certified, worst-case criterion, not of the networks: heuristics that cluster units by tight (complete-linkage) groups instead of filling the defect budget of every node, and choosing as representative not a metric center but the mean or the least-squares solution, trade the guarantee for a smooth, gracefully degrading accuracy-versus-size curve. These \emph{ad hoc} heuristics, although interesting in practice, fall beyond the scope of this paper, which is focused on the theoretical foundations of the method.

\begin{table}[htb]
  \centering
    \begin{tabular}{@{}l l r@{\,}r r@{\,}r r@{}}
    \toprule
    network & $\epsilon$ & \multicolumn{2}{c}{units} & \multicolumn{2}{c}{parameters} & accuracy\\
    \midrule
    MLP      & $0$    & $400$  & ($100\%$) & $267$k   & ($100\%$) & $97.76\%$\\
             & $0.4$  & $305$  & ($76\%$)  & $210$k   & ($79\%$)  & $97.53\%$\\
             & $0.5$  & $169$  & ($42\%$)  & $121$k   & ($45\%$)  & $79.6\%$\\
    \midrule
    CNN      & $0$    & $176$  & ($100\%$) & $80.2$k  & ($100\%$) & $99.09\%$\\
             & $0.4$  & $166$  & ($94\%$)  & $74.7$k  & ($93\%$)  & $98.89\%$\\
             & $0.5$  & $18$   & ($10\%$)  & $2.1$k   & ($2.6\%$) & $14.7\%$\\
    \midrule
    VGG16-BN & $0$    & $5248$ & ($100\%$) & $15.25$M & ($100\%$) & $94.16\%$\\
             & $0.30$ & $3610$ & ($69\%$)  & $8.38$M  & ($55\%$)  & $91.53\%$\\
             & $0.35$ & $3233$ & ($62\%$)  & $6.75$M  & ($44\%$)  & $90.18\%$\\
             & $0.40$ & $2493$ & ($48\%$)  & $4.32$M  & ($28\%$)  & $10.0\%$\\
    \bottomrule
    \end{tabular}
    \caption{Certified $\epsilon$-approximate $M$-fibrations.}
\label{tab:compression-results}
\end{table}

\section{Reproducibility}

For reproducibility, we provide a self-contained Python library that implements the algorithms described in this paper; furthermore, the library contains a script that is able to take any ONNX neural network and compress it using the $\epsilon$-approximate $M$-fibration algorithm, testing it for accuracy and reporting the results. The library is available at \texttt{\url{https://github.com/boldip/mfib}}.

\section{Conclusions}

We introduced a general framework to study fibrations of weighted graphs, and we proved that the minimum base of a weighted graph is unique up to isomorphism. We also introduced the notion of $\epsilon$-approximate fibration, and we showed that it can be used to compress neural networks. 

\section*{\texorpdfstring{\chg{Author contributions}}{Author contributions}}

\chg{The theory of $M$-fibrations presented in this paper (the notions of $M$-graph, $M$-fibration and $\epsilon$-approximate $M$-fibration, and their properties, Sections~\ref{sec:notations}\mbox{--}\ref{sec:approx-equitable}) is due to P.B. The application of this theory to the compression of neural networks (Sections~\ref{sec:application}\mbox{--}\ref{sec:exp}) is due to all authors: it builds on~\cite{VePRFS} and on~\cite{VePHM}, which established $\epsilon$-fibers in weighted graphs, the compression rule based on mean weights, and the application of these ideas to neural network compression, and whose approach was available to P.B. as a private communication prior to publication.}

\section*{Acknowledgement}

\chg{We thank Ian Stewart and Sebastiano Vigna for many fruitful discussions on the topic of this paper.}

\bibliography{biblio,extra}

\end{document}